\documentclass[conference]{IEEEtran}
\IEEEoverridecommandlockouts

\usepackage{cite}
\usepackage{hyperref}
\usepackage{booktabs}
\usepackage{graphicx}
\usepackage{subcaption}
\usepackage{bbm} 
\usepackage{amsmath,amssymb,amsfonts}
\usepackage{algorithmic}
\usepackage{textcomp}
\usepackage{mathtools}
\usepackage{amsthm} 
\usepackage{cleveref}
\usepackage{xcolor}
\def\BibTeX{{\rm B\kern-.05em{\sc i\kern-.025em b}\kern-.08em
    T\kern-.1667em\lower.7ex\hbox{E}\kern-.125emX}}

\usepackage{stfloats}

\begin{document}

\newcommand{\vecX}{\mathbf{x}}
\newcommand{\vecY}{\mathbf{y}}
\newcommand{\vecH}{\mathbf{h}}
\newcommand{\vecT}{\boldsymbol{\theta}}

\newcommand{\matD}{\mathbf{D}}
\newcommand{\matG}{\mathbf{G}}
\newcommand{\matX}{\mathbf{X}}
\newcommand{\matZ}{\mathbf{Z}}
\newcommand{\matR}{\mathbf{R}}
\newcommand{\spaceG}{\mathcal{G}_{\mathrm{rsto}}}
\newcommand{\spaceGhat}{\widehat{\mathbf{G}}}

\theoremstyle{plain}
\newtheorem{theorem}{Theorem}[section]
\newtheorem{proposition}[theorem]{Proposition}
\newtheorem{lemma}[theorem]{Lemma}
\newtheorem{corollary}[theorem]{Corollary}
\theoremstyle{definition}
\newtheorem{definition}[theorem]{Definition}
\newtheorem{assumption}[theorem]{Assumption}
\theoremstyle{remark}
\newtheorem{remark}[theorem]{Remark}


\title{CLaST: Context-aware Contrastive VAE for Probabilistic Time Series Forecasting\\

}
\author{\IEEEauthorblockN{Anonymous}}

\author{
\IEEEauthorblockN{1\textsuperscript{st} Alexander Marusov\IEEEauthorrefmark{1}}
\IEEEauthorblockA{
\textit{Applied AI Institute}\\
Moscow, Russia} 
\\
\IEEEauthorblockN{1\textsuperscript{st} Aleksandr Yugay\IEEEauthorrefmark{1}}
\IEEEauthorblockA{
\textit{Applied AI Institute}\\
Moscow, Russia\\
}
\and
\IEEEauthorblockN{1\textsuperscript{st} Dmitry Anikin\IEEEauthorrefmark{1}}
\IEEEauthorblockA{
\textit{Applied AI Institute}\\
Moscow, Russia}
\\
\IEEEauthorblockN{2\textsuperscript{nd} Petr Sokerin}
\IEEEauthorblockA{
\textit{Applied AI Institute}\\
Moscow, Russia\\
}
\\
\IEEEauthorblockN{4\textsuperscript{th} Alexey Zaytsev}
\IEEEauthorblockA{
\textit{Applied AI Institute}\\
Moscow, Russia\\
}
\and
\IEEEauthorblockN{1\textsuperscript{st} Vitaliy Pozdnyakov\IEEEauthorrefmark{1}}
\IEEEauthorblockA{
\textit{Applied AI Institute}\\
Moscow, Russia\\
}
\\
\IEEEauthorblockN{3\textsuperscript{rd} Ilya Kuleshov}
\IEEEauthorblockA{
\textit{Applied AI Institute}\\
Moscow, Russia\\
}

\thanks{\IEEEauthorrefmark{1}These authors contributed equally to this work.}
}


\maketitle

\begin{abstract}
Probabilistic forecasting models are widely used for time series forecasting in domains such as energy systems, finance, medicine, and transportation. In recent years, deep generative models have shown strong results on probabilistic forecasting, yet many conventional approaches struggle to capture internal temporal dependencies, leading to latent representations with limited expressive power. To address this limitation, we propose \textit{CLaST}, a VAE framework for probabilistic multivariate time series forecasting. Unlike existing generative models, CLaST learns embeddings that preserve contextual similarity between observations through our contrastive loss function. Experiments across nine widely adopted benchmarks demonstrate that CLaST consistently surpasses strong baseline methods. 
In short-term forecasting tasks, our approach achieves improvements of up to $16.4\%$ in CRPS and $14.4\%$ in NMAE over the second-best method. Furthermore, in long-term prediction CLaST attains superior overall performance, exceeding the second-best method by up to $48.6\%$ and $25.1\%$ in CRPS and NMAE, respectively.
\end{abstract}


\section{Introduction}
\label{sec:introduction}
Time series analysis presents a unique set of challenges: temporal dependencies, non-stationarity, multimodality, missing data, and the need to quantify predictive uncertainty. A central task in this domain is time-series forecasting, in which models are required to infer future dynamics from imperfect and often volatile historical data. In a wide range of applications, leveraging deep generative models has consistently improved predictive performance. Recurrent (DeepAR~\cite{salinas2020deepar}), attention-based (PatchTST~\cite{nie2023patchtst}), diffusion (TSDiff-Cond~\cite{kollovieh2023predict}), and flow-based (TFM~\cite{zhang2024trajectory}) approaches offer strong predictive performance. VAE-based methods like $K^2$VAE~\cite{wu2025k2vae} and LaST~\cite{wang2022learning} further structure latents using dynamical systems or trend/seasonal decomposition.

However, learning informative latent representations remains a significant challenge. While estimating mutual information (MI) between input observations and their representations offers a promising approach to enhance embedding informativeness, existing neural MI estimators—such as MINE~\cite{belghazi2018mutual}, CLUB~\cite{pmlr-v119-cheng20b}, and STUB~\cite{wang2022learning}—are computationally intensive, prone to high variance, and exhibit slow convergence in high-dimensional settings~\cite{shin2024experimental, guo2022tight, chan2019neural}. Such instability undermines the optimization process and ultimately restricts the quality of the learned representations.

Furthermore, current approaches often fail to capture latent temporal dependencies embedded within the contextual structure of sequential data. Here \textit{contextual similarity} reflects how closely two observations align in terms of their semantics. For instance, July 1 and July 2 temperature readings share a summer heat contextual regime, whereas July 1 and January 1 belong to distinct seasonal regimes. Crucially, this notion diverges from conventional statistical correlation: variables may co-vary due to shared confounders yet remain contextually distinct. For example, electricity demand and ice cream sales often peak simultaneously in summer, but they describe fundamentally different physical processes and thus occupy separate contextual domains.


To address these limitations, we propose \textit{CLaST}, a variational autoencoder framework for probabilistic time series forecasting that explicitly preserves contextual structure. To adequately capture contextual similarity, we consider a class of time series in which the covariance between observations depends solely on the temporal lag $\delta$ (i.e., $\mathrm{cov}(\vecX_t, \vecX_{t+\delta}) = f(\delta)$), while a time-varying mean maintains overall non-stationarity. Although classic statistical literature has implicitly addressed such processes, to the best of our knowledge, they lack a unified nomenclature. We therefore introduce the term \textit{Lag-Invariant Non-stationary Time Series (LINTS) process}. Unlike strict stationarity assumptions that frequently limit applicability to real-world data, the LINTS formulation retains realistic temporal dynamics through its evolving mean while imposing only a single, interpretable constraint on the dependence structure. Under the LINTS process our approach leverages similarity matrices---a cornerstone of contrastive self-supervised learning~\cite{Balestriero22}---to model the degree of closeness between observations in time. Specifically, we constrain the latent similarity matrices to mirror the structural patterns inherent in the input sequences. Extending the LaST architecture, we replace conventional mutual information (MI) estimators, which are prone to instability, with a contrastive-based penalty that preserves contextual similarity between objects. This formulation simultaneously enhances training stability and facilitates a more robust capture of underlying similarity dependencies.

The main contributions of this work are as follows:  
\begin{itemize}
    \item \textbf{CLaST: Context-aware Contrastive VAE for Probabilistic Time Series Forecasting.} To enhance probabilistic forecasting performance, we introduce a variational autoencoder (VAE) framework grounded in a novel objective function that explicitly accounts for the contextual similarity between samples. Our method ensures that the learned embeddings preserve meaningful sequential dependencies while effectively capturing intrinsic temporal relationships.
    
    \item \textbf{Theoretical properties of our loss function.} We derive the proposed objective function analytically under the LINTS process. Notably, Theorem~\ref{thm:toeplitz} establishes the theoretical optimality of our loss function within the specified class of similarity matrices. By formalizing lag-structured priors within the optimization objective, the loss injects structural knowledge into the learning pipeline. Given sufficiently high-capacity encoders, this formulation compels the model to extract representations that preserve the lag-aware information of the input time series.
    


    \item \textbf{Empirical validation}. Extensive experiments on both short- and long-term forecasting across nine established datasets from diverse domains, including Electricity, ETT, and Weather, demonstrate that CLaST generalizes effectively across heterogeneous time-series modalities. For the short-term setting it consistently outperforms strong baselines, achieving up to $16.4\%$ relative improvement in CRPS and $14.4\%$ in NMAE. Regarding the long-term predictions our approach surpasses second-best result up to $48.6\%$ and $25.1\%$ in CRPS and NMAE correspondingly. We provide the code implementation of the proposed method:\href{https://anonymous.4open.science/r/CLaST-3407/README.md}{https://anonymous.4open.science/r/CLaST-3407/README.md}.
\end{itemize}

\section{Related Work}
\label{sec:related_work}

\paragraph{Classic and deep learning methods.} Classic approaches (ARIMA~\cite{box1970distribution}, VAR~\cite{biller2003modeling}, SVR~\cite{cao2003support}, exponential smoothing~\cite{winters1960forecasting}) struggle with high-dimensional, nonlinear dynamics~\cite{jin2024survey}. Deep learning models address this: RNNs like LSTM~\cite{Hochreiter1997long} enable probabilistic forecasting (e.g., DeepAR~\cite{salinas2020deepar}), while Transformers like PatchTST~\cite{nie2023patchtst} capture long-range dependencies via patch-based attention. Conversely, linear models like DLinear~\cite{Zeng2022AreTE} remain strong baselines by decomposing series into trend and seasonal components.

\paragraph{Variational Autoencoders.} VAEs~\cite{kingma2013auto} offer probabilistic forecasting but often yield unstructured latents. Structured variants address this: LaST~\cite{wang2022learning} factorizes latents into trend/seasonal components with MI regularization, while $K^2$VAE~\cite{wu2025k2vae} combines VAEs with Koopman operators and Kalman filtering to explicitly model temporal dynamics and improve long-horizon accuracy.

\paragraph{Ordinary Differential Equations.} Continuous-time models (Neural ODE~\cite{chen2018neural}, Latent ODE~\cite{rubanova2019latent}, GRU-ODE~\cite{de2019gru}) handle irregular sampling naturally but incur high computational costs and lack latent regularization. Flow-based alternatives like Trajectory Flow Matching (TFM)~\cite{zhang2024trajectory} efficiently learn conditional trajectory derivatives via continuous normalizing flows without stochastic simulation. DeNOTS~\cite{kuleshov2026denots} introduces a negative feedback mechanism to stabilize vector fields, thereby enabling the modeling of long-term dependencies.

\paragraph{Diffusion models.} Diffusion~\cite{ho2020denoising} generates probabilistic forecasts by reversing a noise process conditioned on past observations. While effective for multi-horizon forecasting~\cite{rasul2021autoregressive} and imputation~\cite{tashiro2021csdi}, they are computationally intensive and lack latent interpretability. TSDiff-Cond~\cite{kollovieh2023predict} improves scalability by integrating S4 layers with convolutional channel mixing in an SSSD-based architecture.



Following the literature review, we chose LaST~\cite{wang2022learning}, $K^2$VAE~\cite{wu2025k2vae},
DeNOTS~\cite{kuleshov2026denots},
DeepAR~\cite{salinas2020deepar}, TFM~\cite{zhang2024trajectory}, TSDiff~\cite{kollovieh2023predict}, and PatchTST~\cite{nie2023patchtst} as baseline models due to their strong reported performance across their respective categories.

\section{Method}
\label{sec:method}

To present our proposed method, we first start with the problem statement. Then we outline the overall pipeline of the CLaST approach, followed by a detailed description of its key components. 
The notation appears throughout the text as needed for clarity. 
For a complete list of notation used throughout the paper, we refer the reader to the online documentation available on our \href{https://anonymous.4open.science/r/CLaST-3407/docs/notations.png}{GitHub repository}.

\subsection{Problem Statement.} We study probabilistic forecasting for multivariate time series. Let $\mathbf{x}_t \in \mathbb{R}^F$ denote the observations of $F$ variables at time $t$. Given a historical window $\mathbf{x}_{t-N+1:t}$ of length $N$, the task is to predict the future sequence $\mathbf{x}_{t+1:t+L}$ over a horizon $L$. In the probabilistic setting, the model outputs a predictive distribution $p(\mathbf{x}_{t+1:t+L} \mid \mathbf{x}_{t-N+1:t})$, which captures both future values and their uncertainty.

\subsection{CLaST}\label{method:semantic_loss}

We now introduce \emph{CLaST (Contrastive LaST)}, a generative--contrastive model for probabilistic time series forecasting.
CLaST builds upon the LaST architecture~\cite{wang2022learning} while fundamentally revising its latent regularization strategy.

To address data non-stationarity, LaST employs dedicated seasonal and trend encoders and decoders, providing informative representations for each component. Thus, LaST separately produces latent representations for trend $\matZ^t$ and seasonality $\matZ^s$, along with a predictor module that uses the Discrete Fourier Transform (DFT) for seasonal forecasting. 

A central component of LaST is its loss function. The evidence lower bound (ELBO) term $\mathcal{L}_{\mathrm{ELBO}}$ ensures that the learned trend and seasonal embeddings are sufficiently informative to both reconstruct the input time series and enable accurate predictions. The authors employ a slightly modified MINE-based mutual information estimator to maximize the amount of information preserved from the original input $\matX$ in the trend and seasonal representations, $\matZ^t$ and $\matZ^s$. The corresponding objectives are denoted as $I_{\text{MINE}}(\matX, \matZ^t)$ and $I_{\text{MINE}}(\matX, \matZ^s)$, respectively. Meanwhile, the term $I_{\text{STUB}}(\matZ^s, \matZ^t)$ promotes disentanglement between the trend and seasonal representations, encouraging them to capture complementary, non-overlapping information. The final loss function for LaST is defined in eq.~\ref{method:last_loss}. 

\begin{equation}\label{method:last_loss}
\begin{aligned}
\mathcal{L}_{\mathrm{LaST}}
= {} \mathcal{L}_{\mathrm{ELBO}}
     + I_{\mathrm{MINE}}(\matX, \matZ^{s})
     & + I_{\mathrm{MINE}}(\matX, \matZ^{t}) \\
     & - I_{\mathrm{STUB}}(\matZ^{s}, \matZ^{t}) .
\end{aligned}
\end{equation}

However, as discussed earlier, MI estimations could be unstable in practice. 
We replace the MI-based objective with a contextual similarity alignment loss that encourages latent representations to preserve the contextual structure of the input time series. The central idea is to align the \textit{estimated similarity matrix} $\widehat{\matG}(\matZ)$ with the data-induced similarity matrix $\matG = \matG(\matX)$, which is called \textit{ground truth similarity matrix}. Here, $\widehat{\matG}(\matZ)$ is computed from the trend and seasonal representations $\matZ^t$ and $\matZ^s$, while $\matG$ serves as the  similarity matrix. 

Although the similarity matrices estimated from trend embeddings, $\widehat{\matG}(\matZ^t)$, and seasonal embeddings, $\widehat{\matG}(\matZ^s)$, both approximate the same ground-truth matrix $\matG$, the corresponding representations remain separable in the embedding space, as illustrated in our experiments\footnote{\url{https://anonymous.4open.science/r/CLaST-3407/figs/trend_seasonal.png}}. This separation arises because the LaST framework explicitly models trend and seasonal components as distinct components.

All other parts include the LaST’s core architecture, in particular, separate trend and seasonal latents, and the DFT-based predictor would remain the same. 
The full CLaST objective with three terms is then given by
\begin{equation}\label{method:our_loss}
\mathcal{L}_{\text{CLaST}} =
\mathcal{L}_{\text{ELBO}} 
+ \bigl\lVert \matG - \widehat{\matG} (\matZ^t)\bigr\lVert^2_{\mathcal{F}} 
+ \bigl\lVert \matG - \widehat{\matG}(\matZ^s)\bigr\lVert^2_{\mathcal{F}}.
\end{equation}
Here,
\begin{equation}\label{method:our_regularizer}
\bigl\lVert \matG - \widehat{\matG} (\matZ) \bigr\rVert^2_{\mathcal{F}} = \sum_{i,j} \bigl(g_{ij} - \widehat{g}_{ij}(f_{\boldsymbol{\theta}})\bigr)^2.
\end{equation}
Detailed definitions and discussions of these estimated and ground truth similarity matrices are postponed to Section~\ref{method:ssl_background}. 

\subsection{Method Overview}\label{method:pipeline}
In Section~\ref{method:ssl_background}, we formally define the ground-truth $\matG$ and estimated similarity matrices $\widehat{\matG}$ that underpin our loss function. 

Our approach for constructing the ground truth similarity matrix $\matG$ from time series is described in Section~\ref{method:gt_matrix}. We start from a correlation matrix $\mathbf{R}$, which captures the statistical dependencies in the data. Since correlations do not directly reflect similarity relationships, we introduce a mapping that transforms the correlation space into a similarity space, yielding the  matrix $\matG$ with elements $g_{ij}$ presented in eq.~\eqref{eq:target-G}.

In Section~\ref{method:estimated_matrix}, we derive a closed-form expression for the estimated similarity matrix $\widehat{\matG}$ by formulating and solving a corresponding optimization problem. In~\Cref{thm:toeplitz} we define the elements $\widehat{g}_{ij}(f_{\boldsymbol{\theta}})$ of $\widehat{\matG}$.
We prove that this closed-form solution is the unique global minimizer.

Substituting the ground truth matrix $\matG$ (eq.~\eqref{eq:target-G}) and the estimated matrix $\widehat{\matG}(\matZ)$ (Thm.~\ref{thm:toeplitz}) into the  equation~\eqref{method:our_regularizer} yields the final loss function~\eqref{method:our_loss}.

\subsection{Theoretical Background}\label{method:ssl_background}

To explain our method, we consider a multivariate time series consisting of \(N\) observations \(\{\vecX_i\}_{i=1}^{N}\), where each \(\vecX_i \in \mathbb{R}^F\) represents the \(F\)-dimensional observation at time step \(i\). The time series can thus be represented as a matrix \(\matX \in \mathbb{R}^{N \times F}\).

To utilize a loss function based on pairwise comparisons, we introduce two complementary matrices: a \emph{ground truth similarity} matrix and an \emph{estimated similarity matrix} that are core concepts in contrastive self-supervised learning~\cite{Balestriero22}. 

The \emph{ground truth similarity matrix} $\matG = \{ g_{ij} \}_{i,j=1}^{N} \in \mathbb{R}^{N \times N}$ encodes prior knowledge about the contextual structure of the time series. Each entry $g_{ij}$ represents the desired level of similarity closeness between observations collected at time moments $i \in [1; N]$ and $j \in [1; N]$. This matrix, therefore, defines the target relational structure that the learned representations are expected to preserve.

The \emph{estimated similarity matrix} $\widehat{\matG}(\matZ) = \{ \hat{g}_{ij} \}_{i,j=1}^{N} \in \mathbb{R}^{N \times N}$, on the other hand, captures the similarities inferred by the model itself. The matrix $\widehat{\matG}(\matZ)$ is computed from the embedding matrix $\matZ = \{ \mathbf{z}_i \}_{i=1}^{N} \in \mathbb{R}^{N \times h}$, where each representation $\mathbf{z}_i = f_{\theta}(\vecX_i) \in \mathbb{R}^h$ is obtained by applying an encoder function $f_{\theta}$ to the input signal $\vecX_i$. As a result, $\widehat{\matG}$ reflects how close the samples are in the learned representation space.

By comparing $\widehat{\matG}$ with the ground truth matrix $\matG$, we obtain a natural objective for training the encoder. Minimizing the discrepancy between these two matrices encourages the learned embeddings to preserve the contextual relationships specified by the prior knowledge encoded in $\matG$:
\begin{equation}\label{method:ssl_loss}
\sum_{i,j} (g_{ij} - \widehat{g}_{ij}(f_{\boldsymbol{\theta}}))^2  \rightarrow \min_{\vecT}.
\end{equation}
We use objective function~\ref{method:ssl_loss} as our penalty in eq.~\eqref{method:our_regularizer} for trend and seasonal components. 
Defining this penalty requires specifying the ground truth matrix $\matG$ (Section~\ref{method:gt_matrix}) and the estimated matrix $\widehat{\matG}$ (Section~\ref{method:estimated_matrix}).

\subsection{Ground Truth Similarity Matrix $\matG$}\label{method:gt_matrix}

\subsubsection{Correlation As A Contextual Proxy}

To define the  matrix $\matG$, we will rely on statistical correlations between observations. In particular, we consider a multivariate time series $\{\vecX_i\}_{i=1}^{N}$ modeled as a \textit{Lag-Invariant Non-stationary Time Series (LINTS)}. In this framework, the inter-observation covariance depends exclusively on the temporal lag:
$\mathrm{cov}(\vecX_i, \vecX_{i+\delta}) = f(\delta)$.
While the dependence kernel is shift-invariant, the process itself remains non-stationary due to its time-varying mean.


The sample correlation matrix $\matR$ is constructed as follows.
\begin{enumerate}
    \item \textbf{Centering and normalizing.} For each time point $i \in [1;N]$, subtract a location estimate $\overline{\mathbf{x}}_i = \frac{1}{F} \sum_{j=1}^{F} x_{i}^{(j)}$ to obtain the centered vector, then normalize it to unit Euclidean length:
    \[
    \tilde{\mathbf{x}}_i = \frac{\mathbf{x}_i - \overline{\mathbf{x}}_i}{\|\mathbf{x}_i - \overline{\mathbf{x}}_i\|_2}.
    \]
    This yields observations $\tilde{\mathbf{x}}_i$ that lie on the unit sphere in $\mathbb{R}^F$.
    
    \item \textbf{Computing pairwise similarities.} The correlation matrix $\widetilde{\matR}$ is obtained as the Gram matrix of $\widetilde{\matX} = [\tilde{\mathbf{x}}^{\top}_1, \ldots, \tilde{\mathbf{x}}^{\top}_N]$:
    \[
    \widetilde{\matR} = \widetilde{\matX} \widetilde{\matX}^{\top} \in \mathbb{R}^{N \times N}.
    \]

    \item \textbf{Symmetric Toeplitz structure.} For each lag $\delta \in \{0, \ldots, N - 1\}$ we compute the median empirical correlation
    \[
    r_{\delta} = \operatorname{median} \left\{ \widetilde{r}_{ij} : |i-j|=\delta \right\}.
    \]
    Replacing each entry $\widetilde{r}_{ij}$ by $r_{|i-j|}$ yields a symmetric Toeplitz matrix $\matR$ whose $(i, j)$-entry depends only on the time lag $|i-j|$:
    \[
    r_{ij} = r_{|i - j|}, \qquad i,j = 1, \ldots, N.
    \]
\end{enumerate}
Thus the entire correlation structure is summarized by the $N$ numbers $r_0, r_1, \dots, r_{N - 1}$, with $r_0 = 1$ since $\matR_{ii} = 1$. The Toeplitz property reflects that covariance depends only on lag, while the symmetry $\widetilde{\matR}_{ij} = \widetilde{\matR}_{ji}$ follows from the definition of correlation. 

We now define a  matrix \(\mathbf{G}\) derived from the correlation matrix \(\mathbf{R}\).  
Each entry \(g_{ij} \in [0,1]\) quantifies the degree to which observations collected at time points \(i\) and \(j\) play similar roles within the underlying dynamical process.  
Because the correlations \(r_{ij}\) lie in \([-1; 1]\), we need a mapping \(m: [-1;1] \to [0;1]\) that converts correlation values into probabilities of similarity. Following~\cite{Balestriero22}, we set the main diagonal of $\matG$ to zero, as these entries correspond to the trivial self-similarity of each sample.

We seek a principled mapping from correlations to similarity scores. \Cref{thm:sigmoid} is derived using classical Bayesian decision theory under Gaussian class-conditional assumptions. We show that for LINTS-process, the Bayes-optimal posterior mapping from the observed lag-$\tau$ correlation to contextual similarity is logistic (sigmoid). For simplicity, Theorem~\ref{thm:sigmoid} is stated for the univariate case.

\begin{theorem}[Optimal mapping from feature to similarity space]
\label{thm:sigmoid}
Let $\{x_i\}_{i = 1}^{N}$ be a LINTS-process: covariance depends only on the fixed lag $\delta \ge 0$, so $\mathrm{cov}(x_i, x_{i + \delta}) = f(\delta)$.
Let $Y_{\delta} \in \{0, 1\}$ be a similarity indicator, and let
$\rho_y \in (-1, 1)$ denote the population lag-$\delta$ correlation
associated with $Y_{\delta} = y$.

Assume:
\begin{align*}
\hypertarget{ass:A1}{\text{(A1)}}\quad &
\text{there exists continuously differentiable function }\\ 
& g \in C^1((-1; 1)) \text{ such that }
z_{\delta} := g(r_{\delta}) \text{ satisfies} \\
& p(z_{\delta} \mid Y_{\delta}=y)
=
\frac{1}{\sqrt{2\pi\sigma^2}}
\exp\!\left(
-\frac{(z_{\delta}-\mu_y)^2}{2\sigma^2}
\right), \\
& \mu_y := g(\rho_y), \quad
\sigma^2 \text{ independent of } y; \\[0.5mm]
\label{ass:a2}
\hypertarget{ass:A2}{\text{(A2)}} \quad &
\rho_0 \neq \rho_1.
\end{align*}

Then, in the small-correlation regime $|r_{\delta}|\ll1$,
\begin{equation}
P(Y_{\delta}=1 \mid r_{\delta})
=
\sigma(\alpha r_{\delta} + \beta),
\end{equation}
for some $\alpha,\beta \in \mathbb{R}$, where
$\sigma(x)=(1+e^{-x})^{-1}$.
\end{theorem}

\begin{remark}[Example of a valid transformation]
A classical choice satisfying Assumption \hyperlink{ass:A1}{A1} is the Fisher
$z$-transform
\[
g(r) = \frac12 \log\frac{1+r}{1-r},
\]
for which, under mild mixing conditions, $g(r_{\delta})$ is asymptotically
Gaussian with variance independent of the true correlation.
This motivates the modeling assumption in \hyperlink{ass:A1}{(A1)}.
\end{remark}

The proof of the~\Cref{thm:sigmoid} can be found in the Appendix~\ref{app:sigmoid}.

\begin{figure*}[t]
    \centering
    \includegraphics[width=0.8\linewidth]{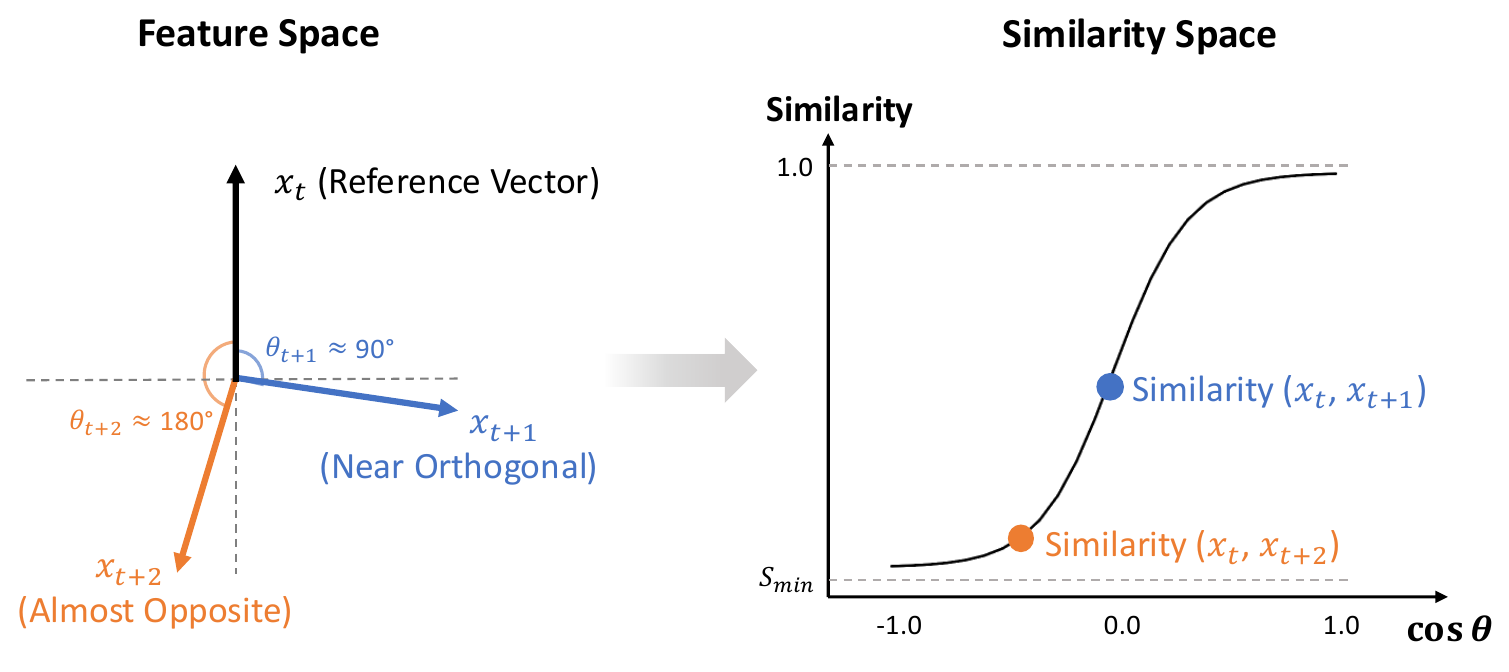}
    \caption{\textbf{Connection between the original feature space and the contextual similarity space.} In the feature space, $x_t$ is the observation at the current time step, and $x_{t+1}$ and $x_{t+2}$ are the inputs at the subsequent time steps. In our formulation, the correlation coefficient $r$ between normalized observations corresponds to the cosine of the angle between them ($r=\cos\theta$). Correlations are mapped to contextual similarity scores via a sigmoid with a lower bound $s_{\min}$, yielding a smooth notion of contextual proximity.}
    \label{fig:data_semantic_space}
\end{figure*}

\subsubsection{Target contextual Similarity}
The intuition underlying the connection between the data space and the similarity space, formalized in \Cref{thm:sigmoid}, is illustrated in Figure~\ref{fig:data_semantic_space}. Accordingly, the entries of the true dependency matrix $\matG$ are defined using a sigmoid function as follows:
\begin{equation}
g_{ij} =
\begin{cases}
s_{\min} + (1 - s_{\min})\,\sigma(\alpha r_{|i-j|}), & i \neq j, \\[4pt]
0, & i = j,
\end{cases}
\label{eq:target-G}
\end{equation}
where \(s_{\min} \in (0, 0.5)\) specifies a lower bound on contextual affinity and $\alpha \in [1,10]$ controls the sensitivity of dependency to variations in correlation.

We introduce $s_{\min}$ to reflect the assumption that even maximally negatively correlated vectors still share structural information and should not have zero contextual similarity.
Our ablation studies~\ref{app:ablation_study} demonstrate that the sigmoid function provides the optimal mapping from correlation space to similarity space, a result validated both theoretically and empirically.

\subsection{Estimated Similarity Matrix $\widehat{\matG}$}\label{method:estimated_matrix}

Now we need to derive an expression of the estimated similarity matrix $\widehat{\matG}$. To do this, we consider the optimization problem of learning an estimated similarity matrix
\( \widehat{\matG} \in \mathbb{R}^{N \times N} \) from a symmetric distance matrix
\( \matD = \{d_{ij}\}_{i,j=1}^{N} \in \mathbb{R}^{N \times N} \), where \( d_{ii} = 0 \). Here $d_{ij} = d(z_i, z_j)$ denotes the distance (e.g. euclidean) between representations $z_i$ and $z_j$ at time moments $i$ and $j$.
The matrix \( \widehat{\matG} \) is assumed to be \emph{Toeplitz}, i.e., its entries depend only on the temporal lag
\( |i-j| \).

Following~\cite{Balestriero22}, we cast the problem of estimating the similarity matrix $\widehat{\matG}$ as a Laplacian estimation task in graph signal processing~\cite{dong2016learning, kalofolias2016learn}. We therefore obtain $\widehat{\matG}$ by solving an optimization problem specifically tailored to time-series data, restricting the search space to $\mathcal{G}$, the set of symmetric, non-negative Toeplitz matrices with a zero main diagonal. 

\begin{equation}
\label{eq:general_problem}
\begin{aligned}
\mathrm{Tr}(\matD \widehat{\matG}) &+ \mathcal{R}_{\log}(\widehat{\matG}) \rightarrow \min_{\widehat{\matG} \in \mathcal{G}}, \\
\end{aligned}
\end{equation}
where \( \mathcal{R}_{\log}(\widehat{\matG}) = \tau \sum_{i \neq j}\hat{g}_{ij}\bigl(\ln(\hat{g}_{ij}) - 1\bigr)\) 
is a logarithmic regularizer.


\begin{theorem}
\label{thm:toeplitz}
(Necessary and sufficient conditions). Let $\mathbf{D}$ be the matrix of pairwise distances, and the regularizer $\mathcal{R_{\log}} (\widehat{\matG}) = \tau \sum_{i \neq j} \hat{g}_{ij} (\ln(\hat{g}_{ij}) - 1)$. 
The matrix $\widehat{\mathbf{G}}$ with elements 
\begin{align*}
    \widehat{g}_{ij} 
    = \exp\!\left( -\dfrac{\phi_{|i-j|}}{\tau} \right)
    \mathbbm{1}_{\{i\neq j\}},
\end{align*}
with $\phi_k = \frac{1}{N - k} \sum_{|i-j| = k} d_{ij}$ 
is the unique global minimizer for the optimization problem~\eqref{eq:general_problem}.
\end{theorem}

The detailed derivation and proof of \Cref{thm:toeplitz}
are given in Appendix~\ref{app:toeplitz_without_constraints}. 

\section{Experimental Results}
\label{sec:exp_results}

\subsection{Experiments Setup}
\paragraph{Datasets} We evaluate our models on nine publicly available datasets that span several domains: ETT (ETTh1, ETTh2, ETTm1, ETTm2)~\cite{haoyietal-informer-2021}, Traffic~\cite{lai2018modeling}, Electricity\footnote{\url{https://archive.ics.uci.edu/dataset/321/electricityloaddiagrams20112014}, accessed in January 2026.}, Weather\footnote{\url{https://www.bgc-jena.mpg.de/wetter/}, accessed in January 2026.}. Detailed information can be found in the Appendix~\ref{app:datasets}.

\paragraph{Metrics.}
We evaluate the forecasts using two commonly used metrics: the Continuous Ranked Probability Score (CRPS) and the Normalized Mean Absolute Error (NMAE) that correspond to the quality of uncertainty estimation and forecasting correspondingly. The detailed descriptions of the metrics are in the Appendix~\ref{app:metrics}.

\begin{table*}[!htbp]
\centering
\caption{\textbf{Comparison of short-term probabilistic time series forecasting performance} ($N = 96$, $L=24$) across various real-world datasets. Lower CRPS and NMAE scores indicate superior forecasting accuracy. Reported results show the \textbf{mean values and standard errors}, computed from \textbf{three independent runs} involving model retraining and evaluation. \textbf{Boldface} highlights the best-performing method, while \underline{underlining} denotes the second-best result.}
\label{tab:short_term}
\resizebox{\textwidth}{!}{
\begin{tabular}{llccccccccc}
\toprule
Dataset & Metric & CLaST (ours) & LaST & $K^2$VAE & DeepAR & TFM & TSDiff & PatchTST & DeNOTS \\
\midrule
Electricity 
& CRPS & \underline{$0.089 \pm 0.000$} & $0.144 \pm 0.021$ & $\mathbf{0.082 \pm 0.000}$ & $0.117 \pm 0.001$ & $0.150 \pm 0.011$ & $0.115 \pm 0.006$ & $0.116 \pm 0.002$ & $0.201 \pm 0.002$ \\
& NMAE & \underline{$0.108 \pm 0.000$} & $0.166 \pm 0.022$ & $\mathbf{0.107 \pm 0.000}$ & $0.151 \pm 0.002$ & $0.201 \pm 0.014$ & $0.150 \pm 0.008$ & $0.116 \pm 0.002$ & $0.274 \pm 0.001$ \\
\midrule
Traffic 
& CRPS & \underline{$0.203 \pm 0.001$} & $0.309 \pm 0.011$ & $\mathbf{0.164 \pm 0.001}$ & $0.207 \pm 0.001$ & $0.321 \pm 0.013$ & $0.208 \pm 0.004$ & $0.256 \pm 0.001$ & $0.386 \pm 0.019$ \\
& NMAE & \underline{$0.241 \pm 0.000$} & $0.360 \pm 0.010$ & $\mathbf{0.211 \pm 0.002}$ & $0.264 \pm 0.002$ & $0.426 \pm 0.022$ & $0.253 \pm 0.004$ & $0.256 \pm 0.001$ & $0.500 \pm 0.032$ \\
\midrule
ETTh1 
& CRPS & $\mathbf{0.220 \pm 0.002}$ & $0.254 \pm 0.004$ & \underline{$0.222 \pm 0.002$} & $0.384 \pm 0.071$ & $0.540 \pm 0.018$ & $0.373 \pm 0.041$ & $0.294 \pm 0.006$ & $0.503 \pm 0.008$ \\
& NMAE & $\mathbf{0.267 \pm 0.003}$ & $0.289 \pm 0.005$ & \underline{$0.283 \pm 0.006$} & $0.498 \pm 0.088$ & $0.623 \pm 0.026$ & $0.492 \pm 0.037$ & $0.294 \pm 0.006$ & $0.645 \pm 0.017$ \\
\midrule
ETTh2 
& CRPS & $\mathbf{0.316 \pm 0.004}$ & $0.460 \pm 0.058$ & \underline{$0.372 \pm 0.005$} & $1.168 \pm 0.091$ & $1.738 \pm 0.451$ & $0.989 \pm 0.213$ & $0.508 \pm 0.055$ & $1.920 \pm 0.398$ \\
& NMAE & $\mathbf{0.371 \pm 0.004}$ & $0.546 \pm 0.077$ & \underline{$0.405 \pm 0.005$} & $1.485 \pm 0.138$ & $1.892 \pm 0.384$ & $1.281 \pm 0.245$ & $0.508 \pm 0.055$ & $2.207 \pm 0.421$ \\
\midrule
ETTm1 
& CRPS & $\mathbf{0.180 \pm 0.004}$ & $0.217 \pm 0.004$ & \underline{$0.186 \pm 0.005$} & $0.359 \pm 0.008$ & $0.556 \pm 0.038$ & $0.299 \pm 0.036$ & $0.247 \pm 0.005$ & $0.561 \pm 0.048$ \\
& NMAE & $\mathbf{0.218 \pm 0.004}$ & $0.257 \pm 0.008$ & \underline{$0.228 \pm 0.003$} & $0.473 \pm 0.005$ & $0.634 \pm 0.038$ & $0.396 \pm 0.059$ & $0.247 \pm 0.005$ & $0.721 \pm 0.050$ \\
\midrule
ETTm2 
& CRPS & $\mathbf{0.229 \pm 0.000}$ & \underline{$0.274 \pm 0.004$} & $0.314 \pm 0.027$ & $0.999 \pm 0.182$ & $1.784 \pm 0.121$ & $0.468 \pm 0.066$ & $0.369 \pm 0.022$ & $1.591 \pm 0.889$ \\
& NMAE & $\mathbf{0.271 \pm 0.002}$ & $0.325 \pm 0.007$ & \underline{$0.316 \pm 0.013$} & $1.273 \pm 0.201$ & $1.904 \pm 0.151$ & $0.630 \pm 0.087$ & $0.369 \pm 0.022$ & $1.815 \pm 0.906$ \\
\midrule
Weather 
& CRPS & $\mathbf{0.296 \pm 0.011}$ & $0.538 \pm 0.057$ & $0.683 \pm 0.011$ & $0.360 \pm 0.025$ & $1.008 \pm 0.322$ & \underline{$0.327 \pm 0.060$} & $0.403 \pm 0.035$ & $0.434 \pm 0.022$ \\
& NMAE & $\mathbf{0.345 \pm 0.016}$ & $0.676 \pm 0.090$ & $0.491 \pm 0.014$ & $0.425 \pm 0.017$ & $1.206 \pm 0.403$ & $0.412 \pm 0.093$ & \underline{$0.403 \pm 0.035$} & $0.562 \pm 0.030$ \\
\midrule
\end{tabular}
}
\end{table*}

\begin{table*}[!t]
\centering
\caption{\textbf{Comparison of long-term probabilistic time series forecasting performance} ($N = 96, L = 720$) across various real-world datasets. Lower CRPS and NMAE scores indicate superior forecasting accuracy. Reported results show the \textbf{mean values and standard errors}, computed from \textbf{three independent runs} involving model retraining and evaluation. \textbf{Boldface} highlights the best-performing method, while \underline{underlining} denotes the second-best result.}
\label{tab:long_term}
\resizebox{\textwidth}{!}{
\begin{tabular}{llcccccccc}
\toprule
Dataset & Metric & CLaST (ours) & LaST & $K^2$VAE & DeepAR & TFM & TSDiff & PatchTST & DeNOTS \\
\midrule
Electricity
& CRPS
& \underline{$0.126 \pm 0.002$}
& $0.146 \pm 0.001$
& $\mathbf{0.115 \pm 0.006}$
& $0.163 \pm 0.009$
& $0.210 \pm 0.017$
& \underline{$0.126 \pm 0.003$}
& $0.156 \pm 0.002$
& $0.186 \pm 0.009$ \\
& NMAE
& \underline{$0.146 \pm 0.002$}
& $0.169 \pm 0.001$
& $\mathbf{0.140 \pm 0.004}$
& $0.204 \pm 0.011$
& $0.292 \pm 0.022$
& $0.165 \pm 0.004$
& $0.156 \pm 0.002$
& $0.186 \pm 0.009$ \\
\midrule
Traffic
& CRPS
& $0.247 \pm 0.001$
& $0.301 \pm 0.003$
& $\mathbf{0.189 \pm 0.002}$
& \underline{$0.220 \pm 0.003$}
& $0.473 \pm 0.018$
& $0.297 \pm 0.027$
& $0.276 \pm 0.002$
& $0.414 \pm 0.143$ \\
& NMAE
& $0.288 \pm 0.001$
& $0.343 \pm 0.002$
& $\mathbf{0.243 \pm 0.000}$
& $0.286 \pm 0.006$
& $0.637 \pm 0.023$
& $0.368 \pm 0.033$
& \underline{$0.276 \pm 0.002$}
& $0.419 \pm 0.151$ \\
\midrule
ETTh1
& CRPS
& \underline{$0.342 \pm 0.005$}
& $0.573 \pm 0.009$
& $\mathbf{0.334 \pm 0.008}$
& $0.690 \pm 0.211$
& $0.590 \pm 0.060$
& $0.575 \pm 0.059$
& $0.511 \pm 0.001$
& $0.555 \pm 0.014$ \\
& NMAE
& $\mathbf{0.379 \pm 0.004}$
& $0.614 \pm 0.010$
& \underline{$0.428 \pm 0.012$}
& $0.830 \pm 0.240$
& $0.839 \pm 0.088$
& $0.701 \pm 0.054$
& $0.511 \pm 0.001$
& $0.750 \pm 0.019$ \\
\midrule
ETTh2
& CRPS
& $\mathbf{0.544 \pm 0.001}$
& $2.531 \pm 0.054$
& \underline{$1.000 \pm 0.235$}
& $2.332 \pm 0.013$
& $2.495 \pm 0.688$
& $2.254 \pm 0.106$
& $1.541 \pm 0.022$
& $2.082 \pm 0.136$ \\
& NMAE
& $\mathbf{0.588 \pm 0.002}$
& $2.602 \pm 0.038$
& \underline{$0.769 \pm 0.083$}
& $2.611 \pm 0.053$
& $3.204 \pm 0.992$
& $2.488 \pm 0.104$
& $1.541 \pm 0.022$
& $2.356 \pm 0.104$ \\
\midrule
ETTm1
& CRPS
& \underline{$0.309 \pm 0.000$}
& $0.421 \pm 0.013$
& $\mathbf{0.266 \pm 0.024}$
& $0.466 \pm 0.073$
& $0.592 \pm 0.010$
& $0.483 \pm 0.020$
& $0.404 \pm 0.029$
& $0.576 \pm 0.009$ \\
& NMAE
& \underline{$0.341 \pm 0.001$}
& $0.487 \pm 0.015$
& $\mathbf{0.328 \pm 0.039}$
& $0.623 \pm 0.069$
& $0.855 \pm 0.025$
& $0.629 \pm 0.025$
& $0.404 \pm 0.029$
& $0.614 \pm 0.010$ \\
\midrule
ETTm2
& CRPS
& $\mathbf{0.508 \pm 0.001}$
& $1.871 \pm 0.241$
& \underline{$0.989 \pm 0.275$}
& $2.064 \pm 0.646$
& $2.899 \pm 0.782$
& $1.642 \pm 0.450$
& $1.240 \pm 0.152$
& $1.931 \pm 0.016$ \\
& NMAE
& $\mathbf{0.534 \pm 0.001}$
& $2.022 \pm 0.251$
& \underline{$0.713 \pm 0.105$}
& $2.388 \pm 0.633$
& $3.783 \pm 0.975$
& $1.915 \pm 0.430$
& $1.240 \pm 0.152$
& $2.008 \pm 0.036$ \\
\midrule
Weather
& CRPS
& \underline{$0.447 \pm 0.001$}
& $0.672 \pm 0.028$
& $0.661 \pm 0.007$
& $0.465 \pm 0.064$
& $1.344 \pm 0.246$
& $\mathbf{0.375 \pm 0.030}$
& $0.685 \pm 0.034$
& $0.803 \pm 0.071$ \\
& NMAE
& \underline{$0.482 \pm 0.001$}
& $0.835 \pm 0.043$
& $0.606 \pm 0.008$
& $0.555 \pm 0.067$
& $1.542 \pm 0.422$
& $\mathbf{0.465 \pm 0.033}$
& $0.685 \pm 0.034$
& $0.865 \pm 0.084$ \\
\midrule
\end{tabular}
}
\end{table*}

\paragraph{Baselines.} We compare our approach against several generative models, including modern $K^2$VAE~\cite{wu2025k2vae}, LaST~\cite{wang2022learning}, TSDiff~\cite{kollovieh2023predict}, and TFM~\cite{zhang2024trajectory}. Additionally, we include PatchTST~\cite{nie2022time}, a point-wise forecasting model, for reference.

\paragraph{Backbone.}
To choose an appropriate backbone encoder, we considered three types of backbone architectures: DLinear, MLP and PatchTST. 
The hyperparameters for the selected backbones are detailed in Appendix~\ref{app:implementation_details}.
For both short- and long-term forecasting tasks, we consistently employ PatchTST as the backbone encoder. 

\subsection{Main Results}

\paragraph{Short-term forecasting ($N=96,L=24$)} For short-term prediction, all experiments are performed with a forecasting horizon of $L=24$ and an context length of $N=96$.

As shown in Table~\ref{tab:short_term}, CLaST delivers the best performance across several datasets, highlighting its effectiveness for probabilistic forecasting. It consistently outperforms the second-best method on datasets ETTh1, ETTh2, ETTm1, ETTm2, and Weather, achieving relative improvements of $0.9–16.4\%$ in CRPS and $4.4–14.4\%$ in NMAE. Our comparative analysis reveals that when PatchTST is trained within the CLaST generative framework, it demonstrably surpasses its pointwise-trained counterpart across every benchmark dataset, as measured by the NMAE. The observed performance gains range from a minimum improvement of $5.86\%$ on the Traffic dataset to a substantial maximum of $26.97\%$ on ETTh2, underscoring the significant advantage conferred by our CLaST generative learning paradigm.


\paragraph{Long-term forecasting ($N=96, L=720$)} As presented in Table~\ref{tab:long_term}, CLaST achieves superior CRPS and NMAE performance on the demanding ETTh2 and ETTm2 datasets, yielding reductions of up to $48.6\%$ in CRPS on ETTh2 and $25.1\%$ in NMAE on ETTm2. Integrating PatchTST into the CLaST framework yields substantial NMAE improvements across most benchmarks, with relative gains ranging from $6.41\%$ on Electricity to $61.84\%$ on ETTh2.

The exceptionally high dimensionality of the Traffic (862 channels) and Electricity (321 channels) datasets renders long-term forecasting particularly challenging. Consequently, we extended our evaluation to two additional energy-domain datasets, Solar and ERCOT~\cite{shchur2025fev}. 
Due to the substantial computational resources already expended on the main Table~\ref{tab:long_term} for long-term experiments running architectures like diffusion-based (TSDiff), attention-based (PatchTST), flow-based (TFM), and ODE-based (e.g. DeNOTS) was not feasible within our remaining budget. We therefore restrict this comparison to the two most relevant VAE-based baselines: $K^2VAE$, the strongest empirical competitor on the main benchmarks, and LaST, the foundational architecture upon which CLaST is built. As demonstrated in Table~\ref{tab:solar_ercot}, CLaST consistently attains the best long-term forecasting results, outperforming the second-best model by margins of up to $12.9\%$ in CRPS and $11\%$ in NMAE. As shown in Figure~\ref{fig:clast-components} our CLasT effectively models both trend and seasonal components.

\begin{table}[!htbp]
\centering
\scriptsize
\caption{\textbf{Comparison of long-term probabilistic time series forecasting performance on Solar and ERCOT datasets} ($N = 96$, $L=720$). Lower CRPS and NMAE scores indicate superior forecasting accuracy. Reported results show the \textbf{mean values and standard errors}, computed from \textbf{three independent runs} involving model retraining and evaluation. \textbf{Boldface} highlights the best-performing method, while \underline{underlining} denotes the second-best result.}
\label{tab:solar_ercot}
\begin{tabular}{llccccccc}
\toprule
Dataset & Metric & CLaST (ours) & K$^2$VAE & LaST \\
\midrule
ERCOT 
& CRPS & $\mathbf{0.080 \pm 0.001}$ & \underline{$0.084 \pm 0.001$} & $0.095 \pm 0.004$ \\
& NMAE & $\mathbf{0.090 \pm 0.002}$ & \underline{$0.103 \pm 0.002$} & $0.106 \pm 0.005$ \\
\midrule
Solar 
& CRPS & $\mathbf{0.529 \pm 0.010}$ & $0.696 \pm 0.020$ & \underline{$0.607 \pm 0.015$} \\
& NMAE & $\mathbf{0.573 \pm 0.016}$ & $0.637 \pm 0.006$ & \underline{$0.644 \pm 0.011$} \\
\bottomrule
\end{tabular}
\end{table}

\paragraph{Ablation studies} To verify that the observed performance gains stem from our regularization strategy rather than the choice of encoder, we conduct additional ablation experiments on three datasets (ETTh1, ETTh2, Weather) using different backbone architectures. With PatchTST as the backbone, CLaST improves both CRPS and NMAE by up to 46.8\% and 45.6\%, respectively~\footnote{\url{https://anonymous.4open.science/r/CLaST-3407/tables/last_clast_patchtst.md}}. When using the original FeedNet backbone from LaST, the gains reach 46.17\% (CRPS) and 52.37\% (NMAE)~\footnote{\url{https://anonymous.4open.science/r/CLaST-3407/tables/last_clast_feednet.md}}.

We conducted a detailed ablation study to understand the model analysis. In particular to examine the effect of hyperparameter choices, we present a sensitivity analysis in Appendix~\ref{app:sensitivity_analysis}. 
Appendix~\ref{app:sense_hidden_dim} studies the effect of the latent space dimensionality while Appendix~\ref{app:sense_hist_length} analyzes the impact of the context length. Additionally, we study how well temporal structure is preserved in the latent space. Our results in Appendix~\ref{app:temporal_structure} indicate that the learned embeddings retain temporal dependencies. The implementation details are provided in the Appendix~\ref{app:model_analysis}.

\section{Discussion}
\textbf{How the proposed loss enhances forecasting.} 
Accurate probabilistic forecasting requires capturing diverse temporal dynamics and quantifying uncertainty. When latent representations collapse, decoders produce over-smoothed predictions and degraded CRPS. Our loss prevents this by structuring the latent space to preserve lag-decay correlations, keeping temporally distinct inputs separable and enabling sharper, better-calibrated forecasts (Fig.~\ref{fig:similarity_heatmaps}, Fig.~\ref{fig:clast-components}). The resulting gains in CRPS and NMAE directly reflect this principled temporal regularisation.

\textbf{Theoretical limitations.} 
CLaST assumes inter-observation covariance depends solely on temporal lag, not absolute time. This aligns with many real-world domains where dependency structures remain stable despite non-stationary means: climatological series with long-term trends, financial indicators under trend-stationary models~\cite{hamilton2020time}, and industrial sensor data with invariant physical dynamics~\cite{williams2006gaussian}.

\textbf{Practical limitations.} 
The contextual similarity alignment loss incurs $\mathcal{O}(N^2)$ complexity due to full similarity-matrix construction. In practice, this overhead is modest: LaST trains $\sim$28\% faster per epoch than CLaST~\footnote{\url{https://anonymous.4open.science/r/CLaST-3407/tables/time.md}}, yet CLaST delivers substantially better performance. As we mentioned earlier in our main results with PatchTST backbone, CLaST improves CRPS/NMAE by up to 46.8\%/45.6\%
Using the FeedNet, gains reach 46.17\%/52.37\
This marginal computational cost is strongly justified by the consistent, backbone-agnostic improvements in predictive quality.

\section{Conclusion}
This paper presented \textbf{CLaST}, a Context-aware VAE framework for probabilistic time series forecasting. Standard generative objectives often fail to encourage informative latent representations, particularly in sequential settings. CLaST mitigates this issue by introducing a natural theoretically-grounded loss function that accounts for contextual similarity between observations. It transfers temporal structure from the input space to the latent space using theoretically grounded loss functions, ensuring that the embeddings preserve contextual structure of dependent data.

Across a range of benchmarks covering diverse domains such as Electricity, ETT, and Weather, the proposed method demonstrates consistent and meaningful gains over strong baselines. These improvements reflect CLaST's ability to capture intrinsic temporal relationships while maintaining the flexibility of latent-variable models. In short-term forecasting tasks, our approach achieves improvements of up to $16.4\%$ in CRPS and $14.4\%$ in NMAE. Furthermore, for the long-term forecasting CLaST attains superior overall performance, exceeding the second-best method by up to $48.6\%$ and $25.1\%$ in CRPS and NMAE, respectively, indicating enhanced uncertainty quantification and reduced forecasting error without undermining prior modeling advances.

\bibliographystyle{IEEEtran}
\bibliography{references}

@book{hamilton2020time,
  title={Time series analysis},
  author={Hamilton, James D},
  year={2020},
  publisher={Princeton university press}
}

@inproceedings{kuleshov2026denots,
    title={{DeNOTS}: Stable Deep Neural {ODE}s for Time Series},
    author={Ilya Kuleshov and Evgenia Romanenkova and Vladislav Andreevich Zhuzhel and Galina Boeva and Evgeni Vorsin and Alexey Zaytsev},
    booktitle={The Fourteenth International Conference on Learning Representations (ICLR)},
    year={2026},
    url={https://openreview.net/forum?id=SFoDJZ1sSk}
}

@book{williams2006gaussian,
  title={Gaussian processes for machine learning},
  author={Williams, Christopher KI and Rasmussen, Carl Edward},
  volume={2},
  number={3},
  year={2006},
  publisher={MIT press Cambridge, MA}
}

@article{wu2025k2vae,
  title={K2VAE: A Koopman-Kalman Enhanced Variational AutoEncoder for Probabilistic Time Series Forecasting},
  author={Wu, Xingjian and Qiu, Xiangfei and Gao, Hongfan and Hu, Jilin and Yang, Bin and Guo, Chenjuan},
  journal={arXiv preprint arXiv:2505.23017},
  year={2025}
}

@article{zhang2024probts,
  title={ProbTS: Benchmarking point and distributional forecasting across diverse prediction horizons},
  author={Zhang, Jiawen and Wen, Xumeng and Zhang, Zhenwei and Zheng, Shun and Li, Jia and Bian, Jiang},
  journal={Advances in Neural Information Processing Systems},
  volume={37},
  pages={48045--48082},
  year={2024}
}

@inproceedings{zhou2021informer,
  title={Informer: Beyond efficient transformer for long sequence time-series forecasting},
  author={Zhou, Haoyi and Zhang, Shanghang and Peng, Jieqi and Zhang, Shuai and Li, Jianxin and Xiong, Hui and Zhang, Wancai},
  booktitle={Proceedings of the AAAI conference on artificial intelligence},
  volume={35},
  number={12},
  pages={11106--11115},
  year={2021}
}

@article{wu2021autoformer,
  title={Autoformer: Decomposition transformers with auto-correlation for long-term series forecasting},
  author={Wu, Haixu and Xu, Jiehui and Wang, Jianmin and Long, Mingsheng},
  journal={Advances in neural information processing systems},
  volume={34},
  pages={22419--22430},
  year={2021}
}

@article{box1970distribution,
  title={Distribution of residual autocorrelations in autoregressive-integrated moving average time series models},
  author={Box, George EP and Pierce, David A},
  journal={Journal of the American statistical Association},
  volume={65},
  number={332},
  pages={1509--1526},
  year={1970},
  publisher={Taylor \& Francis}
}

@article{biller2003modeling,
  title={Modeling and generating multivariate time-series input processes using a vector autoregressive technique},
  author={Biller, Bahar and Nelson, Barry L},
  journal={ACM Transactions on Modeling and Computer Simulation (TOMACS)},
  volume={13},
  number={3},
  pages={211--237},
  year={2003},
  publisher={ACM New York, NY, USA}
}

@article{cao2003support,
  title={Support vector machine with adaptive parameters in financial time series forecasting},
  author={Cao, Li-Juan and Tay, Francis Eng Hock},
  journal={IEEE Transactions on neural networks},
  volume={14},
  number={6},
  pages={1506--1518},
  year={2003},
  publisher={IEEE}
}

@article{winters1960forecasting,
  title={Forecasting sales by exponentially weighted moving averages},
  author={Winters, Peter R},
  journal={Management science},
  volume={6},
  number={3},
  pages={324--342},
  year={1960},
  publisher={INFORMS}
}

@article{jin2024survey,
  title={A survey on graph neural networks for time series: Forecasting, classification, imputation, and anomaly detection},
  author={Jin, Ming and Koh, Huan Yee and Wen, Qingsong and Zambon, Daniele and Alippi, Cesare and Webb, Geoffrey I and King, Irwin and Pan, Shirui},
  journal={IEEE Transactions on Pattern Analysis and Machine Intelligence},
  year={2024},
  publisher={IEEE}
}

@article{salinas2020deepar,
  title={DeepAR: Probabilistic forecasting with autoregressive recurrent networks},
  author={Salinas, David and Flunkert, Valentin and Gasthaus, Jan and Januschowski, Tim},
  journal={International journal of forecasting},
  volume={36},
  number={3},
  pages={1181--1191},
  year={2020},
  publisher={Elsevier}
}

@article{kingma2013auto,
  title={Auto-Encoding Variational Bayes},
  author={Kingma, Diederik P and Welling, Max},
  journal={arXiv e-prints},
  pages={arXiv--1312},
  year={2013}
}

@article{wang2022learning,
  title={Learning latent seasonal-trend representations for time series forecasting},
  author={Wang, Zhiyuan and Xu, Xovee and Zhang, Weifeng and Trajcevski, Goce and Zhong, Ting and Zhou, Fan},
  journal={Advances in Neural Information Processing Systems},
  volume={35},
  pages={38775--38787},
  year={2022}
}

@article{chen2018neural,
  title={Neural ordinary differential equations},
  author={Chen, Ricky TQ and Rubanova, Yulia and Bettencourt, Jesse and Duvenaud, David K},
  journal={Advances in neural information processing systems},
  volume={31},
  year={2018}
}

@article{rubanova2019latent,
  title={Latent ordinary differential equations for irregularly-sampled time series},
  author={Rubanova, Yulia and Chen, Ricky TQ and Duvenaud, David K},
  journal={Advances in neural information processing systems},
  volume={32},
  year={2019}
}

@article{de2019gru,
  title={GRU-ODE-Bayes: Continuous modeling of sporadically-observed time series},
  author={De Brouwer, Edward and Simm, Jaak and Arany, Adam and Moreau, Yves},
  journal={Advances in neural information processing systems},
  volume={32},
  year={2019}
}

@article{ho2020denoising,
  title={Denoising diffusion probabilistic models},
  author={Ho, Jonathan and Jain, Ajay and Abbeel, Pieter},
  journal={Advances in neural information processing systems},
  volume={33},
  pages={6840--6851},
  year={2020}
}

@inproceedings{rasul2021autoregressive,
  title={Autoregressive denoising diffusion models for multivariate probabilistic time series forecasting},
  author={Rasul, Kashif and Seward, Calvin and Schuster, Ingmar and Vollgraf, Roland},
  booktitle={International conference on machine learning},
  pages={8857--8868},
  year={2021},
  organization={PMLR}
}

@article{tashiro2021csdi,
  title={Csdi: Conditional score-based diffusion models for probabilistic time series imputation},
  author={Tashiro, Yusuke and Song, Jiaming and Song, Yang and Ermon, Stefano},
  journal={Advances in neural information processing systems},
  volume={34},
  pages={24804--24816},
  year={2021}
}

@article{hochreiter1997long,
  title={Long short-term memory},
  author={Hochreiter, Sepp and Schmidhuber, J{\"u}rgen},
  journal={Neural computation},
  volume={9},
  number={8},
  pages={1735--1780},
  year={1997},
  publisher={MIT Press}
}

@article{kollovieh2023predict,
  title={Predict, refine, synthesize: Self-guiding diffusion models for probabilistic time series forecasting},
  author={Kollovieh, Marcel and Ansari, Abdul Fatir and Bohlke-Schneider, Michael and Zschiegner, Jasper and Wang, Hao and Wang, Yuyang Bernie},
  journal={Advances in Neural Information Processing Systems},
  volume={36},
  pages={28341--28364},
  year={2023}
}

@article{zhang2024trajectory,
  title={Trajectory flow matching with applications to clinical time series modelling},
  author={Zhang, Xi Nicole and Pu, Yuan and Kawamura, Yuki and Loza, Andrew and Bengio, Yoshua and Shung, Dennis and Tong, Alexander},
  journal={Advances in Neural Information Processing Systems},
  volume={37},
  pages={107198--107224},
  year={2024}
}

@article{nie2022time,
  title={A Time Series is Worth 64Words: Long-term Forecasting with Transformers},
  author={Nie, Y},
  journal={arXiv preprint arXiv:2211.14730},
  year={2022}
}

@article{zamo2018estimation,
  title={Estimation of the continuous ranked probability score with limited information and applications to ensemble weather forecasts},
  author={Zamo, Micha{\"e}l and Naveau, Philippe},
  journal={Mathematical Geosciences},
  volume={50},
  number={2},
  pages={209--234},
  year={2018},
  publisher={Springer}
}

@inproceedings{Balestriero22,
 author    = {Balestriero, R. and LeCun, Y.},
 title     = {Contrastive and non-contrastive self-supervised learning recover global and local spectral embedding methods},
 year      = {2022},
 pages     = {26671--26685},
 booktitle     = {Advances in Neural Information Processing Systems (NeurIPS 2022)}
}

@inproceedings{nie2023patchtst,
  title     = {A Time Series is Worth 64 Words: Long-term Forecasting with Transformers},
  author    = {Nie, Yuqi and
               H. Nguyen, Nam and
               Sinthong, Phanwadee and 
               Kalagnanam, Jayant},
  booktitle = {International Conference on Learning Representations},
  year      = {2023}
}

@article{Zeng2022AreTE,
  title={Are Transformers Effective for Time Series Forecasting?},
  author={Ailing Zeng and Muxi Chen and Lei Zhang and Qiang Xu},
  journal={arXiv preprint arXiv:2205.13504},
  year={2022}
}

@article{dong2016learning,
  title={Learning Laplacian matrix in smooth graph signal representations},
  author={Dong, Xiaowen and Thanou, Dorina and Frossard, Pascal and Vandergheynst, Pierre},
  journal={IEEE Transactions on Signal Processing},
  volume={64},
  number={23},
  pages={6160--6173},
  year={2016},
  publisher={IEEE}
}

@inproceedings{kalofolias2016learn,
  title={How to learn a graph from smooth signals},
  author={Kalofolias, Vassilis},
  booktitle={Artificial intelligence and statistics},
  pages={920--929},
  year={2016},
  organization={PMLR}
}

@article{chan2019neural,
  title={Neural entropic estimation: A faster path to mutual information estimation},
  author={Chan, Chung and Al-Bashabsheh, Ali and Huang, Hing Pang and Lim, Michael and Tam, Da Sun Handason and Zhao, Chao},
  journal={arXiv preprint arXiv:1905.12957},
  year={2019}
}

@article{shin2024experimental,
author = {Shin, Donghoon and Kim, Hyung},
year = {2024},
month = {01},
pages = {},
title = {An experimental study of neural estimators of the mutual information between random vectors modeling power spectrum features},
volume = {2024},
journal = {EURASIP Journal on Advances in Signal Processing},
doi = {10.1186/s13634-023-01092-1}
}

@article{guo2022tight,
  title={Tight mutual information estimation with contrastive fenchel-legendre optimization},
  author={Guo, Qing and Chen, Junya and Wang, Dong and Yang, Yuewei and Deng, Xinwei and Huang, Jing and Carin, Larry and Li, Fan and Tao, Chenyang},
  journal={Advances in Neural Information Processing Systems},
  volume={35},
  pages={28319--28334},
  year={2022}
}

@inproceedings{belghazi2018mutual,
  title={Mutual information neural estimation},
  author={Belghazi, Mohamed Ishmael and Baratin, Aristide and Rajeshwar, Sai and Ozair, Sherjil and Bengio, Yoshua and Courville, Aaron and Hjelm, Devon},
  booktitle={International conference on machine learning},
  pages={531--540},
  year={2018},
  organization={PMLR}
}

@InProceedings{pmlr-v119-cheng20b,
  title = 	 {{CLUB}: A Contrastive Log-ratio Upper Bound of Mutual Information},
  author =       {Cheng, Pengyu and Hao, Weituo and Dai, Shuyang and Liu, Jiachang and Gan, Zhe and Carin, Lawrence},
  booktitle = 	 {Proceedings of the 37th International Conference on Machine Learning},
  pages = 	 {1779--1788},
  year = 	 {2020},
  editor = 	 {III, Hal Daumé and Singh, Aarti},
  volume = 	 {119},
  series = 	 {Proceedings of Machine Learning Research},
  month = 	 {13--18 Jul},
  publisher =    {PMLR},
  url = 	 {https://proceedings.mlr.press/v119/cheng20b.html},
}

@inproceedings{haoyietal-informer-2021,
  author    = {Haoyi Zhou and
               Shanghang Zhang and
               Jieqi Peng and
               Shuai Zhang and
               Jianxin Li and
               Hui Xiong and
               Wancai Zhang},
  title     = {Informer: Beyond Efficient Transformer for Long Sequence Time-Series Forecasting},
  booktitle = {The Thirty-Fifth {AAAI} Conference on Artificial Intelligence, {AAAI} 2021, Virtual Conference},
  volume    = {35},
  number    = {12},
  pages     = {11106--11115},
  publisher = {{AAAI} Press},
  year      = {2021},
}

@inproceedings{lai2018modeling,
  title={Modeling long-and short-term temporal patterns with deep neural networks},
  author={Lai, Guokun and Chang, Wei-Cheng and Yang, Yiming and Liu, Hanxiao},
  booktitle={The 41st international ACM SIGIR conference on research \& development in information retrieval},
  pages={95--104},
  year={2018}
}

@inproceedings{akiba2019optuna,
  title={Optuna: A next-generation hyperparameter optimization framework},
  author={Akiba, Takuya and Sano, Shotaro and Yanase, Toshihiko and Ohta, Takeru and Koyama, Masanori},
  booktitle={Proceedings of the 25th ACM SIGKDD international conference on knowledge discovery \& data mining},
  pages={2623--2631},
  year={2019}
}

@techreport{shchur2025fev,
  title={fev-bench: A Realistic Benchmark for Time Series Forecasting},
  author={Shchur, Oleksandr and Ansari, Abdul Fatir and Turkmen, Caner and Stella, Lorenzo and Erickson, Nick and Guerron-Quintana, Pablo and Bohlke-Schneider, Michael and Wang, Yuyang},
  year={2025},
  institution={Boston College Department of Economics}
}

\appendix

Below, we provide the complete mathematical derivations, lemma proofs, and reductions used to obtain the final closed-form solution.

\subsection{Optimal mapping from feature to similarity space}\label{app:sigmoid}
\begin{proof}
By Assumptions~\hyperlink{ass:A1}{A1}--\hyperlink{ass:A2}{A2}, $z_{\delta}:=g(r_{\delta})$ is Gaussian with class-conditional means $\mu_y=g(\rho_y)$ ($\mu_0\neq\mu_1$) and common variance $\sigma^2$. The log-likelihood ratio is linear:
\[
\Lambda(z_{\delta}) = \log\frac{p(z_{\delta}\mid Y_{\delta}=1)}{p(z_{\delta}\mid Y_{\delta}=0)}
= \alpha z_{\delta} + \beta', \quad \alpha:=\frac{\mu_1-\mu_0}{\sigma^2}.
\]
By Bayes' rule, the posterior is $P(Y_{\delta}=1\mid z_{\delta}) = \sigma(\alpha z_{\delta}+\beta)$ with $\beta:=\beta'+\log(\pi_1/\pi_0)$. For small $r_{\delta}$, a first-order Taylor expansion $g(r_{\delta})=g(0)+g'(0)r_{\delta}+o(r_{\delta})$ yields
\[
\begin{aligned}
P(Y_{\delta}=1\mid r_{\delta}) = \sigma(\tilde{\alpha} r_{\delta} + \tilde{\beta}) + o(r_{\delta}), \\
\tilde{\alpha} := \alpha g'(0),\; 
\quad \tilde{\beta}:=\beta+\alpha g(0),
\end{aligned}
\]
completing the proof.
\end{proof}

\subsection{Necessary and sufficient conditions}
\label{app:toeplitz_without_constraints}
\begin{proof}
Since $\mathcal{G}$ consists of symmetric, non-negative Toeplitz matrices with zero diagonal, any $\widehat{\matG}\in\mathcal{G}$ is parameterized by $\widehat{g}=(\widehat{g}_1,\dots,\widehat{g}_{N-1})$ with $\widehat{g}_k\ge0$, where $\widehat{g}_{ij}=\widehat{g}_{|i-j|}$. Defining the lag-averaged distance $\phi_k \coloneqq \frac{1}{N-k}\sum_{|i-j|=k} d_{ij}$, problem~\eqref{eq:general_problem} reduces to
\begin{equation}
\label{eq:reduced}
\min_{\widehat{g}_k>0}\;
\mathcal{J}(\widehat{g}) = 2\sum_{k=1}^{N-1}(N-k)\bigl[\phi_k\widehat{g}_k + \tau\widehat{g}_k(\ln\widehat{g}_k-1)\bigr].
\end{equation}
$\mathcal{J}$ is separable and strictly convex on $\{\widehat{g}_k>0\}$ because each term has second derivative $2\tau(N-k)/\widehat{g}_k>0$. Setting $\partial\mathcal{J}/\partial\widehat{g}_k=0$ yields the unique stationary point
\[
\widehat{g}_k^\star = \exp(-\phi_k/\tau), \qquad k=1,\dots,N-1,
\]
which, by strict convexity, is the unique global minimizer of~\eqref{eq:reduced}. The corresponding Toeplitz matrix $\widehat{\matG}$ is therefore the unique solution to~\eqref{eq:general_problem}.

\end{proof}

\subsection{Datasets Description}
\label{app:datasets}
We evaluate on nine standard multivariate time-series benchmarks: ETTh1/2 and ETTm1/2 (hourly/15-min transformer measurements, 2016--2018), Electricity (hourly consumption for 321 users, 2012--2014), Weather (10-min meteorological records, 2020), Traffic (hourly road occupancy across 862 sensors, 2015--2016), Solar (hourly energy consumption and weather data with solar irradiance and meteorological variables), and ERCOT (hourly electricity load across 8 U.S. regions, 2004–2021)~\cite{wu2025k2vae, zhang2024probts, zhou2021informer, wu2021autoformer, shchur2025fev}. These datasets span diverse temporal resolutions, dimensionalities, and seasonal patterns. To encode temporal context, we append 11 auxiliary features: sinusoidal embeddings $(\cos(2\pi t), \sin(2\pi t))$ for year, quarter, month, week, and day cycles, plus a normalized inter-sample interval $(t_i - t_{i-1}) / \Delta t$. All series are chronologically partitioned into 70\% training, 10\% validation, and 20\% testing splits.

\subsection{Evaluation Metrics}
\label{app:metrics}
We report forecast performance using two standard metrics: Normalized Mean Absolute Error (NMAE) for point accuracy and Continuous Ranked Probability Score (CRPS) for probabilistic calibration. NMAE is computed as
\[
\mathrm{NMAE}(x, \hat{x}) = \frac{1}{N} \sum_{n=1}^N \frac{\sum_{i=1}^{F}\lvert x^n_i - \hat{x}^n_i \rvert}{\sum_{i=1}^{F}\lvert x^n_i \rvert},
\]
where $\hat{x}$ denotes the median of 100 sampled trajectories. CRPS measures the discrepancy between the predictive CDF $F(z)$ and the observation $x$:
\[
\mathrm{CRPS} = \int_{\mathbb{R}} \bigl(F(z) - \mathbb{I}\{x \leq z\}\bigr)^2 dz,
\]
approximated via Monte Carlo sampling~\cite{zamo2018estimation}. For both metrics, lower values indicate better performance.

\subsection{Implementation Details}\label{app:implementation_details}
We tune CLaST hyperparameters on the validation set via Optuna~\cite{akiba2019optuna}, searching $\alpha \in [1,10]$, $s_{\min} \in (0,0.5)$, and $\tau \in (0.01,10)$. We evaluate three encoder backbones:

\textbf{DLinear}: We modify the original architecture~\cite{Zeng2022AreTE} to project trend/seasonal components into latent space (dim 64, or 128 for Electricity/Traffic) using a running-mean window of 25.

\textbf{MLP}: A two-layer ReLU MLP with hidden dimension tuned over $\{64,128,256\}$ and fixed latent output dimension 64.

\textbf{PatchTST}: Separate PatchTST encoders~\cite{nie2023patchtst} model trend and seasonal components, each followed by linear heads for latent mean/variance estimation. Patch length and stride are fixed to 16 and 8; latent dimension is 64 (128 for Electricity/Traffic). Remaining hyperparameters are optimized via Optuna\footnote{\url{https://anonymous.4open.science/r/CLaST-3407/tables/patchtst_backbone.md}} or follow the original implementation\footnote{\url{https://github.com/yuqinie98/PatchTST}}.

\section{CLaST Analysis}
\label{app:model_analysis}

\begin{figure*}[!h]
    \centering
    \includegraphics[width=\linewidth]{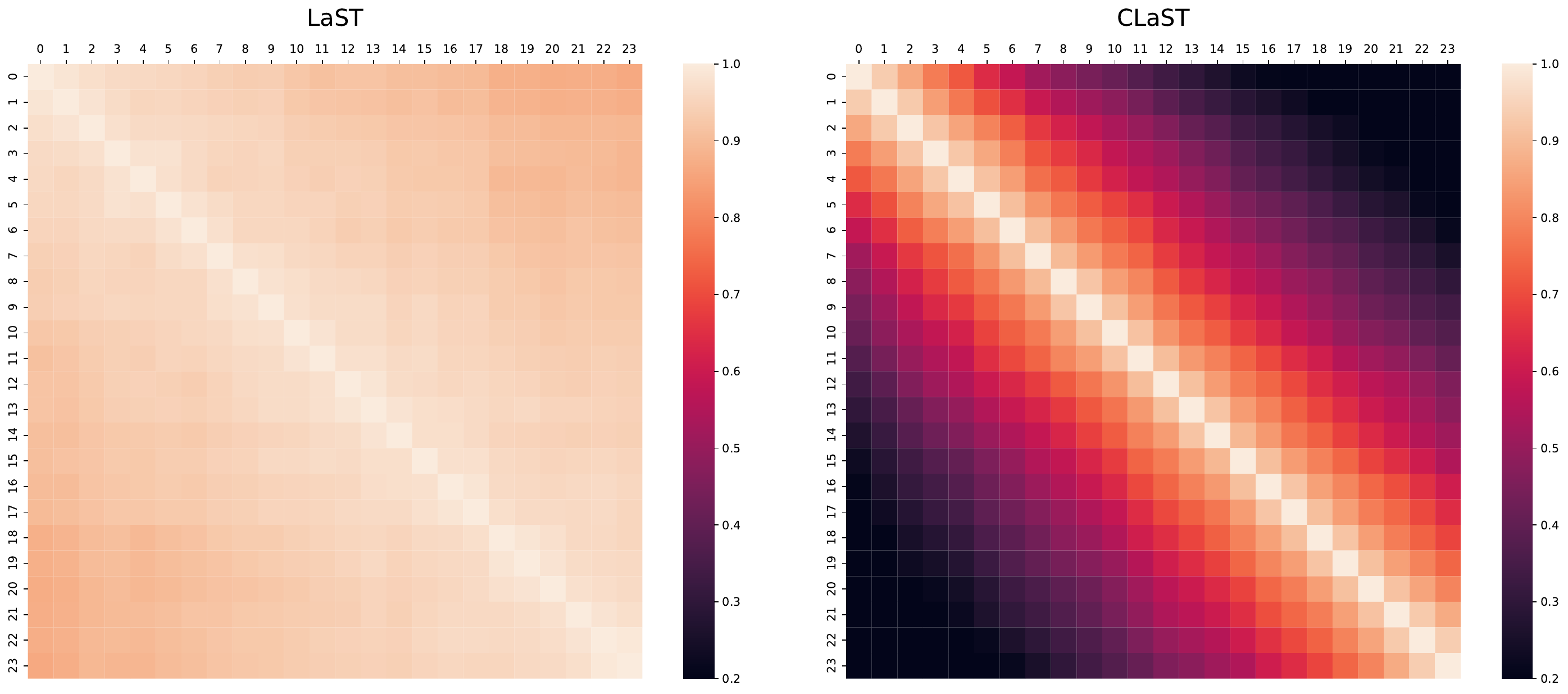}
    \caption{Average cosine similarity between embeddings of input data across time steps for LaST and CLaST on the Weather dataset with context length and horizon $N=L=24$.}
    \label{fig:similarity_heatmaps}
\end{figure*}

\subsection{Mapping functions from correlation to similarity space}
\label{app:ablation_study}

We evaluate the impact of different mappings from the correlation space to the similarity space. Specifically, we compare the theoretically motivated \textit{sigmoid} mapping with alternative choices, including ReLU, absolute value, and a parabolic transformation. As shown in the Table~\ref{tab:mapping}, the sigmoid mapping consistently yields the best performance across both datasets (Electricity and Traffic) according to both CRPS and NMAE.

\begin{table}[!h]
\centering
\caption{\textbf{Impact of mapping functions from correlation to similarity space} on probabilistic forecasting performance for Electricity and Traffic datasets. Lower CRPS and NMAE indicate better performance. Results are reported as mean $\pm$ standard error}.
\label{tab:mapping}
\resizebox{\linewidth}{!}{
\begin{tabular}{llcccc}
\toprule
Dataset & Metric & Sigmoid & ReLU & Absolute & Parabola \\
\midrule
Electricity 
& CRPS 
& $\mathbf{0.107 \pm 0.000}$ 
& $0.111 \pm 0.000$ 
& $0.111 \pm 0.001$ 
& $0.111 \pm 0.001$ \\

& NMAE 
& $\mathbf{0.132 \pm 0.000}$ 
& $0.137 \pm 0.000$ 
& $0.140 \pm 0.001$ 
& $0.139 \pm 0.001$ \\

\midrule
Traffic 
& CRPS 
& $\mathbf{0.270 \pm 0.004}$ 
& $0.274 \pm 0.001$ 
& $0.278 \pm 0.003$ 
& $0.285 \pm 0.012$ \\

& NMAE 
& $\mathbf{0.323 \pm 0.009}$ 
& $0.332 \pm 0.001$ 
& $0.338 \pm 0.004$ 
& $0.345 \pm 0.013$ \\

\bottomrule
\end{tabular}
}
\end{table}

\subsection{Sensitivity Analysis}\label{app:sensitivity_analysis}

\subsubsection{Sensitivity to Latent Space Dimensionality}
\label{app:sense_hidden_dim}

The dimensionality of the latent space is selected based on validation set optimization. However, it is important to assess the model’s sensitivity to this parameter, as a stable model should not exhibit significant performance degradation when using comparable dimensionalities. Our empirical results~\footnote{\url{https://anonymous.4open.science/r/CLaST-3407/tables/dim_abl.png}} demonstrate that CLaST achieves consistent forecasting performance across latent dimensions of 32, 64, 128, and 256.


\subsubsection{Backbone Sensitivity to Context Length}\label{app:sense_hist_length}

All encoders are tested with different context lengths of 24, 48, 96, 192, and 384 time steps. The impact of the encoder choice and the context length is provided in our supplementary materials~\footnote{\url{https://anonymous.4open.science/r/CLaST-3407/tables/backbones.png}}. The Transformer-based CLaST model generally performs best, with the MLP backbone typically ranking second. However, for some long-horizon forecasts--particularly on the Traffic and ETTh1 datasets--the MLP model can outperform the Transformer. On the other hand, for the Weather dataset, the Transformer model achieves better results than the MLP in certain long-horizon settings, while underperforming in others.

\subsection{Temporal Structure in Learned Embeddings}\label{app:temporal_structure}
To better understand the effect of the proposed training objective on the learned representations, we analyze the temporal structure of the resulting embedding space.

Figure \ref{fig:similarity_heatmaps} presents the average pairwise cosine similarity of input embeddings across different temporal positions for LaST and CLaST. Both approaches exhibit a diagonal structure in the resulting similarity matrices, indicating that embeddings corresponding to temporally adjacent timestamps are more similar, while similarity decreases with increasing temporal distance. However, LaST shows uniformly high average cosine similarity across the matrix, suggesting overly smooth representations with limited discriminability between distant timestamps. In contrast, CLaST yields lower and more heterogeneous similarity values, leading to a more pronounced diagonal pattern and improved preservation of temporal structure, due to the contrastive regularization and temporal bias imposed by the training loss.



\begin{figure}[!b]
  \centering

  \begin{subfigure}[t]{\linewidth}
    \centering
    \includegraphics[width=\linewidth]{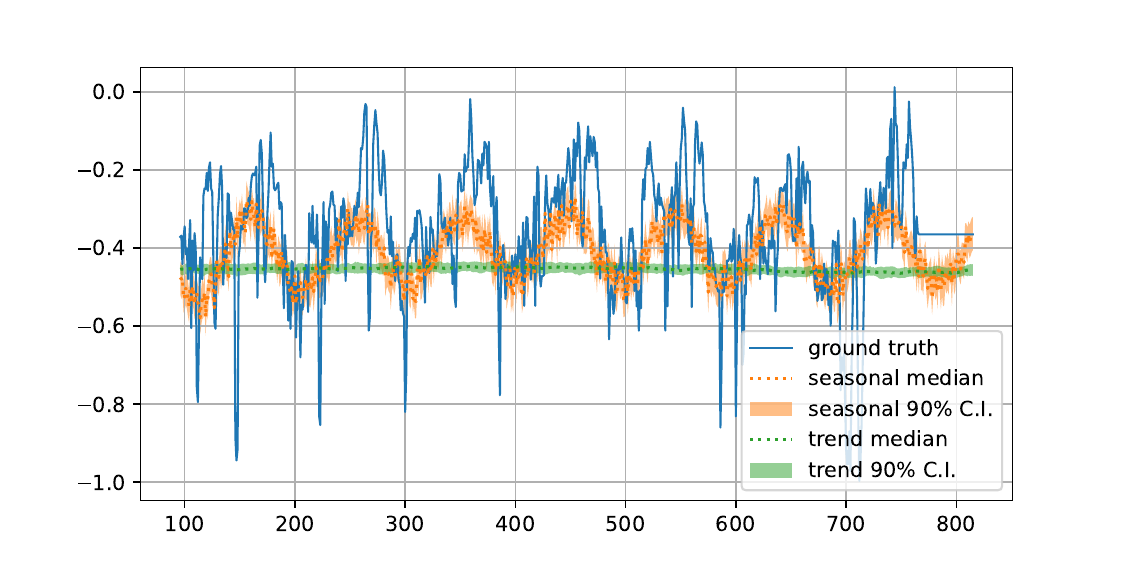}
    \caption{ETTm2.}
    \label{fig:clast-components-ettm2}
  \end{subfigure}


  \begin{subfigure}[t]{\linewidth}
    \centering
    \includegraphics[width=\linewidth]{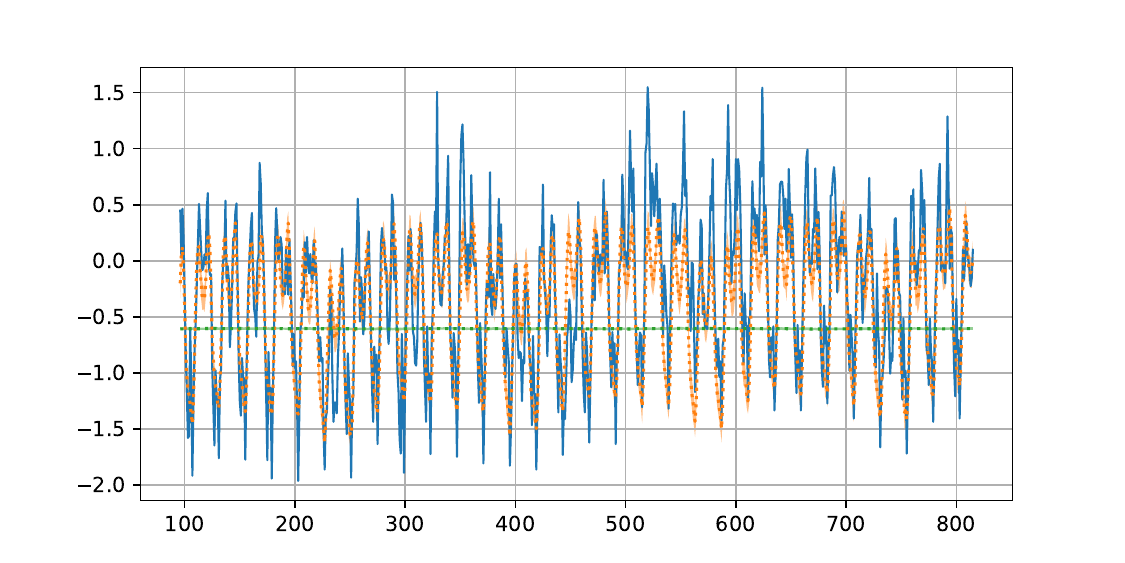}
    \caption{Electricity.}
    \label{fig:clast-components-electricity}
  \end{subfigure}

  \caption{CLaST decomposed forecasts ($N=96$, $L=720$).
    Each figure shows ground truth, seasonal, and trend with 90\% confidence intervals from 100 samples.}
  \label{fig:clast-components}
\end{figure}

\end{document}